\documentclass[sigconf]{acmart}
\AtBeginDocument{%
  \providecommand\BibTeX{{%
    \normalfont B\kern-0.5em{\scshape i\kern-0.25em b}\kern-0.8em\TeX}}}

\copyrightyear{2021}
\acmYear{2021}
\setcopyright{acmlicensed}
\acmConference[CIKM '21]{Proceedings of the 30th ACM International Conference on Information and Knowledge Management}{November 1--5, 2021}{Virtual Event, QLD, Australia}
\acmBooktitle{Proceedings of the 30th ACM International Conference on Information and Knowledge Management (CIKM '21), November 1--5, 2021, Virtual Event, QLD, Australia}
\acmPrice{15.00}
\acmDOI{10.1145/3459637.3482357}
\acmISBN{978-1-4503-8446-9/21/11}
\usepackage{graphicx}
\usepackage{subcaption}
\usepackage{algorithm}
\usepackage{algorithmic}
\usepackage{float}
\usepackage{bbm}
\usepackage{lscape} 
\usepackage{tabularx}

\newtheorem{lemma}{Lemma}
\newtheorem{assumption}{Assumption}
\newtheorem{claim}{Claim}
\newtheorem{theorem}{Theorem}

\newenvironment{proofsketch}{%
  \renewcommand{\proofname}{Proof Sketch}\proof}{\endproof}

\begin{document}

\fancyhead{}
\title{Jointly-Learned State-Action Embedding for Efficient Reinforcement Learning}

\author{Paul J. Pritz}
\email{paul.pritz18@imperial.ac.uk}
\affiliation{%
  \institution{Imperial College London}
  \city{London}
  \country{United Kingdom}
}

\author{Liang Ma}
\email{lma@dataminr.com}
\affiliation{%
  \institution{Dataminr}
  \city{New York}
  \country{USA}}

\author{Kin K. Leung}
\email{kin.leung@imperial.ac.uk}
\affiliation{%
  \institution{Imperial College London}
  \city{London}
  \country{United Kingdom}
}

\renewcommand{\shortauthors}{Pritz and Ma, et al.}

\begin{abstract}
    While reinforcement learning has achieved considerable successes in recent years, state-of-the-art models are often still limited by the size of state and action spaces.
    Model-free reinforcement learning approaches use some form of state representations and the latest work has explored embedding techniques for actions, both with the aim of achieving better generalization and applicability.
    However, these approaches consider only states or actions, ignoring the interaction between them when generating embedded representations.
    In this work, we establish the theoretical foundations for the validity of training a reinforcement learning agent using embedded states and actions.
    We then propose a new approach for jointly learning embeddings for states and actions that combines aspects of model-free and model-based reinforcement learning, which can be applied in both discrete and continuous domains.
    Specifically, we use a model of the environment to obtain embeddings for states and actions and present a generic architecture that leverages these to learn a policy.
    In this way, the embedded representations obtained via our approach enable better generalization over both states and actions by capturing similarities in the embedding spaces.
    Evaluations of our approach on several gaming, robotic control, and recommender systems show it significantly outperforms state-of-the-art models in both discrete/continuous domains with large state/action spaces, thus confirming its efficacy.
\end{abstract}


\maketitle

\section{Introduction}
\label{sec:intro}

\textit{Reinforcement learning} (RL) has been successfully applied to a range of tasks, including challenging gaming scenarios \citep{mnih2015humanlevel}.
However, the application of RL in many real-world domains is often hindered by the large number of possible states and actions these settings present.
For instance, resource management in computing clusters \citep{RL_resource_management_hotnets_2016,google_datacenter_RL_2016}, portfolio management \citep{jian_portfolio_maagement_RL}, and recommender systems \citep{recommender_RL_1_Yu_2019,recommender_RL_Liu_2018} all suffer from extremely large state/action spaces, thus challenging to be tackled by RL.

In this work, we investigate efficient training of reinforcement learning agents in the presence of large state-action spaces, aiming to improve the applicability of RL to real-world domains.
Previous work attempting to address this challenge has explored the idea of learning representations (embeddings) for states or actions, either by adding additional layers to the RL agent's network architecture or by using separately trained embedding models.
Specifically, for state embeddings, using machine learning to obtain meaningful features from raw state representations is a common practice in RL, e.g. through the use of convolutional neural networks for image input \citep{mnih2013atari}.
Previous works such as  \citet{ha_schmidhuber_world_models_2018} have explored the use of environment models, termed \emph{world models}, to learn state representations in a supervised fashion, and several pieces of literature explore state aggregation using bisimulation metrics \citep{castro_2020_scalable_methods_for_state_similarity}.
While for action embeddings, the most recent works by \citet{Tennenholtz2019} and \citet{Chandak2019} propose methods for learning embeddings for discrete actions that can be directly used by an RL agent and improve generalization over actions.
However, these works consider the state representation and action representation as isolated tasks, which ignore the underlying relationships between them.
In this regard, we take a different approach and propose to jointly learn embeddings for states and actions, aiming for better generalization over both states and actions in their respective embedding spaces.

More importantly, even with state or action embedding as in some existing works, there is generally a lack of theoretical guarantee that ensures the learned policy in the embedding space can also achieve optimality in the original problem domain.
To this end, we establish the theoretical foundations showing how the policy learned purely in the state and action embedding space is linked to the optimal policy in the original problem domain.
We then propose an architecture consisting of two models: a model of the environment that is used to generate state and action representations and a model-free RL agent that learns a policy using the embedded states and actions.
One key benefit of this approach is that state and action representations can be learned in a supervised manner, which greatly improves sampling efficiency and potentially enables their use for transfer learning.
In sum, our key contributions are:
\begin{itemize}
    \item We establish the theoretical relationship between learning a policy in the original problem domain and learning an internal policy in embedding space, and prove the equivalence between updates to these two policies.
    \item We formulate an embedding model for states and actions, along with an internal policy $\pi_i$ that leverages the learned state/action embeddings, as well as the corresponding overall policy $\pi_o$ that acts in the original problem domain. We show the existence of an overall policy $\pi_o$ that achieves optimality in the original problem domain.
    \item We present a supervised learning algorithm for the proposed embedding model that can be combined with any policy gradient based RL algorithm.
    \item We evaluate our methodology on  game-based as well as real-world tasks and find that it outperforms state-of-the-art models in both discrete/continuous domains.
\end{itemize}

The remainder of this paper is structured as follows: In Section \ref{sec:background}, we provide some background on RL.
We then give an overview of related work in Section \ref{sec:rel_work}. The proposed methodology is presented in Section \ref{sec:method}, followed by the evaluations in Section \ref{sed:experiments}. Section~\ref{sec:conclusion} concludes the paper.

\section{Background}
\label{sec:background}
We consider an agent interacting with its environment over discrete time steps, where the environment is modelled as a discrete-time \textit{Markov decision process} (MDP), defined by a tuple $(\mathcal{S}, \mathcal{A}, \mathcal{T}, \mathcal{R}, \gamma)$.
$\mathcal{S}$ and $\mathcal{A}$ denote the \textit{state space} and \textit{action space}, respectively.
In this work, we consider both discrete and continuous state and action spaces.
The transition function from one state to another, given an action, is $\mathcal{T} : \mathcal{S} \times \mathcal{A} \mapsto \mathcal{S}$, which may be \emph{deterministic} or \emph{stochastic}.
The agent receives a reward at each time step defined by $\mathcal{R}: \mathcal{S} \times \mathcal{A} \mapsto \mathbb{R} $.
$\gamma \in [0,1]$ denotes the reward discounting factor.
The state, action, and reward at time $t \in \{0, 1, 2, \ldots\}$ are denoted by the \emph{random variables} $S_t$, $A_t$, and $R_t$.
The initial state of the environment comes from an initial state distribution $d_0$.
Thereafter, the agent follows a policy $\pi$, defined as a conditional distribution over actions given states, i.e., $\pi(a | s) = P(A_t = a | S_t = s)$.
The goal of the reinforcement learning agent is to find an optimal policy $\pi^*$ that maximizes the expected sum of discounted accumulated future rewards for a given environment, i.e., 
$\pi^* \in \arg\, \max_{\pi} \mathbb{E}[\sum_{t=0}^{\infty} \gamma^t R_t | \pi]$.
For any policy, we also define the state value function $v^{\pi}(s) = \mathbb{E}[\sum_{k=0}^{\infty} \gamma^k R_{t+k} | \pi, S_t = s]$ and the state-action value function $Q^{\pi} (s, a) = \mathbb{E}[\sum_{k=0}^{\infty} \gamma^k R_{t+k} | \pi, S_t = s, A_t = a]$.

\section{Related Work}
\label{sec:rel_work}

For the application of state embeddings in reinforcement learning, there are two dominant strands of research, namely \emph{world models} and \emph{state aggregation} using bisimulation metrics.
World model approaches train an environment model in a supervised fashion from experience collected in the environment, which is then used to generate compressed state representations \citep{ha_schmidhuber_2018_recurrent_world_models} or to train an agent using the learned world model \citep{ha_schmidhuber_world_models_2018,schmidhuber_2015_learning_to_think}.
Further applications of world models, e.g. for Atari 2000 domains, show that abstract state representations learned via world models can substantially improve sample efficiency \citep{kaiser_2019_model_based_atari,hafner_2020_mastering_atari_with_wrold_models}.
Recent work by \citet{tao2020novelty,ermolov2020latent} demonstrate that state embeddings learned via world models can also be used to guide exploration during training, either based on the error of transition predictions \citep{ermolov2020latent} or based on the distance of state embeddings in latent space \cite{tao2020novelty}.
Similar to this idea, \citet{munk_2016} pre-train an environment model and use it to provide state representations for an RL agent.
Furthermore, \citet{bruin_kober_2018} investigate additional learning objectives to learn state representations, and \citet{FrancoisLavet2019} propose the use of an environment model to generate abstract state representations; their learned state representations are then used by a Q-learning agent.
By using a learned model of the environment to generate abstract states, these approaches capture structure in the state space and reduce the dimensionality of its representation.
In contrast to world models, we consider both states and actions and train an RL agent on the original environment using its embedded representation rather than using a surrogate world model. 
Bisimulation, on the other hand, is a method for aggregating states that are ``behaviorally equivalent'' \citep{li_2006_state_abstraction} and can improve convergence speed by grouping similar states into abstract states.
Several works on bisimulation-based state aggregation, e.g. \citet{ferns_2004_metrics,givan_2003_equivalence_notion_mdp}, present different metrics to measure state similarity.
Furthermore, \citet{zhang_2020_learning_representations} and \citet{castro_2020_scalable_methods_for_state_similarity} propose deep learning methods for generating bisimulation-based state aggregations that scale beyond the tabular methods proposed in several earlier works \citep{castro_2020_scalable_methods_for_state_similarity}.
While there are parallels between bisimulation and our approach, we do not propose the aggregation of states.
Instead, our embedding technique projects states into a continuous state embedding space, similar to \citet{zhang_2020_learning_representations}, where their behavioral similarity is captured by their proximity in embedding space.
Furthermore, our method embeds both states and actions and does not employ an explicit similarity metric such as a bisimulation metric, but instead learns the relationships among different states and actions via an environment model.
State representations are also used in RL-based NLP tasks, such as \citet{narasimhan_NLP_2015}, who jointly train an LSTM-based state representation module and a DQN agent, and \citet{ammanabrolu_NLP_2019}, who propose the use of a knowledge graph based state representation.

In addition to state representations, previous work has explored the use of additional models to learn meaningful action representations. 
In particular, \citet{hasselt_2009} investigate the use of a continuous actor in a policy gradient based approach to solve discrete action MDPs, where the policy is learned in continuous space and actions are discretized before execution in the environment.
\citet{dulac_arnold_large_action_spaces_2015} propose a similar methodology, where a policy is learned in continuous action embedding space and then mapped to discrete actions in the original problem domain.
Both \citet{hasselt_2009} and \citet{dulac_arnold_large_action_spaces_2015} only consider actions and assume that embeddings are known a priori.
\citet{Tennenholtz2019} propose a methodology called Act2Vec, where they introduce an embedding model similar to the Skip-Gram model \citep{mikolov2013_skipgram} that is trained using data from expert demonstrations and then combined with a DQN agent.
One significant drawback of this approach is that information on the semantics of actions has to be injected via expert demonstration and is not learned automatically.
In contrast, \citet{Chandak2019, whitney2019dynamics} propose methods that enable the self-supervised learning of state and action representations from experience collected by the RL agent.
\citet{Chandak2019} use an embedding model that resembles CBOW.
\citet{Chandak2019} only consider action embeddings, while we jointly embed states and actions and improve the performance significantly.
\citet{whitney2019dynamics} on the other hand, present an approach for embedding both states and actions.
However, their approach generates embeddings for sequences of actions, resulting in a temporal abstraction.
They then train a policy on these temporally-abstracted actions.
In contrast to \citet{whitney2019dynamics}, we propose embedding individual states and actions and present theoretical results.
As an alternative to embedding methods, \citet{pazis_parr_2011} and \citet{sallans_hinton_2004} represent actions in binary format, where they learn a value function associated with each bit, and \citet{sharma_2017} use factored actions, where actions are expressed using underlying primitive actions.
Again, these three methods rely on handcrafted action decomposition.

Previous embedding approaches consider either states or actions in isolation or use embeddings for temporal abstraction.
By contrast, in addition to capturing structure in the state or action space respectively, we leverage interdependencies between the two by jointly learning embeddings for them, and improve the quality of embeddings accordingly.

\section{Methodology}
\label{sec:method}

In this section, we propose a model to jointly learn state and action embeddings and present a generic framework that uses these embeddings in conjunction with any policy gradient algorithm.

\subsection{Architecture}
\label{method:architecture}
Our proposed approach has three components: (i) state-action embedding model, (ii) policy gradient based RL agent that takes in the embedded states and outputs actions in action embedding space, and (iii) action mapping function that maps embedded actions to actions in the original problem domain.
\begin{figure*}[t]
    \centering
    \begin{subfigure}[b]{.44\textwidth}
        \centering
        \includegraphics[width=0.8\linewidth]{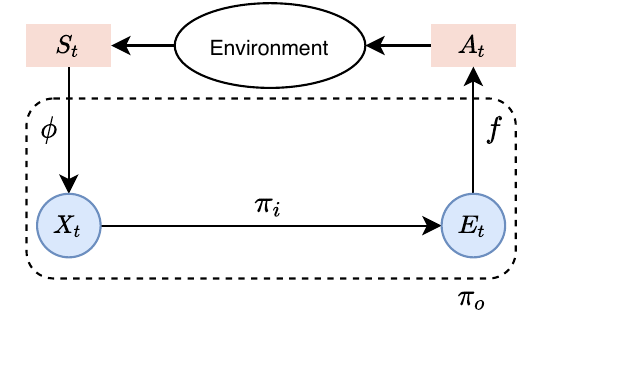}
        \vspace{-3em}
        \caption{}
        \label{fig:combined_model}
    \end{subfigure}%
    \begin{subfigure}[b]{.44\textwidth}
        \centering
        \includegraphics[width=0.8\linewidth]{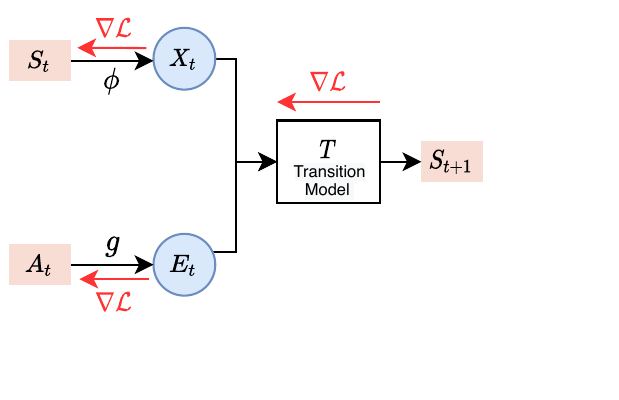}
        \vspace{-3em}
        \caption{}
        \label{fig:embed_model}
    \end{subfigure}%
    \vspace{-0.3em}
    \caption{(a) Overall architecture, using the state embedding function $\phi$, the internal policy $\pi_i$, and the action mapping function $f$.
        (b) Environment model for learning embedding function $\phi$ and the auxiliary action embedding function $g$ (the red arrows denote the gradient flow of the supervised loss in Equation~(\ref{eq:loss_embedding})).}
    \label{fig:model}
\end{figure*}
We define two embedding spaces, i.e., a continuous state embedding space $\mathcal{X} \in \mathbb{R}^m$ and a continuous action embedding space $\mathcal{E} \in \mathbb{R}^d$; states are projected to the state embedding space via a function $\phi: \mathcal{S} \mapsto \mathcal{X}$.
On the other hand, we define a function $f: \mathcal{E} \mapsto \mathcal{A}$ that maps points in the action embedding space to actions in the original problem domain.
With these embedding spaces, we further define an internal policy $\pi_i$ that takes in the embedded state and outputs the embedded action, i.e., $\pi_i: \mathcal{X} \mapsto \mathcal{E}$. Furthermore, the mapping functions $\phi$ and $f$, together with the internal policy $\pi_i$ form the overall policy $\pi_o: \mathcal{S} \mapsto \mathcal{A}$, i.e., $\pi_o$ is characterized by the following three components:
\begin{equation*}
    X_t  = \phi(S_t),
    E_t \sim \pi_i(\cdot | X_t), \text{and }
    A_t = f(E_t),
\end{equation*}
where random variables $S_t \in \mathcal{S}$, $X_t \in \mathcal{X}$, $E_t \in \mathcal{E}$, and $A_t \in \mathcal{A}$ denote the state in the original problem domain, state in the embedding domain, action in the embedding domain, and action in the original domain at time $t$, respectively.
The overall model architecture is illustrated in Figure~\ref{fig:combined_model}.
In Figure~\ref{fig:combined_model}, function $\phi$ is directly learned using an embedding model illustrated in Figure \ref{fig:embed_model} (detailed in Section \ref{method:embedding}).
However, function $f$ is not directly part of this embedding model, but uses the obtained action embeddings to learn a mapping from $e\in\mathcal{E}$ to $a\in\mathcal{A}$ (see Section~\ref{method:algorithm} for details).

For an RL problem, the ultimate goal is to optimize the policy $\pi$ in the original problem domain, i.e., $\pi$ in the $\mathcal{S} \times \mathcal{A}$ space.
For the proposed architecture, the idea is to first establish efficient embeddings for both states and actions in $\mathcal{S}$ and $\mathcal{A}$ via self-supervised learning.
Then we reduce the problem of finding the optimal $\pi^*$ in the original problem domain to the problem of internal policy $\pi_i$ optimization that acts purely in the embedding space. However, it is still unclear how the optimization of $\pi_i$ is related to $\pi^*$. In addition, the overall policy $\pi_o$ relying on $\phi$ and $f$ further complicates the policy learning.
In this regard, we establish the relationships among $\pi$, $\pi_o$, and $\pi_i$ in the next section, which is the foundation to justify the validity of leveraging both state and action embeddings in RL.

\subsection{Theoretical Foundations: Relationships among $\pi$, $\pi_o$, and $\pi_i$}
\label{method:PG}

To understand the relationships among $\pi$, $\pi_o$, and $\pi_i$, we first derive how $\pi_o$ is related to $\pi$, after which we prove the existence of $\pi_o$ that achieves optimality in the original problem domain. For this goal, two further assumptions are required on the nature of the state embedding function $\phi$ and the action mapping function $f$.
\begin{assumption}
    \label{as:act_mapping}
    Given an action embedding $E_t$, $A_t$ is deterministic and defined by a function $f: \mathcal{E} \mapsto \mathcal{A}$, i.e., there exists action $a$ such that $P(A_t = a | E_t = e) = 1$.
\end{assumption}
\begin{assumption}
    \label{as:state_mapping}
    Given a state $S_t$, $X_t$ is deterministic and defined by a function $\phi: \mathcal{S} \mapsto \mathcal{X}$, i.e., there exists state embedding $x$ such that $P(X_t = x | S_t = s) = 1$.
\end{assumption}
We validate Assumption \ref{as:act_mapping}, which defines $f$ as a \emph{many-to-one} mapping, empirically for all experiments conducted in Section \ref{sed:experiments}, and find that no two actions share exactly the same embedded representation, i.e., Assumption \ref{as:act_mapping} holds in practice.
Note that we also assume the Markov property for our environment, which is a standard assumption for reinforcement learning problems.
With slight abuse of notation, we denote the inverse mapping from an action $a \in \mathcal{A}$ to its corresponding points in the embedding space (\emph{one-to-many} mapping) by $f^{-1}(a) := \{e \in \mathcal{E}: f(e) = a\}$.
\begin{lemma}
    \label{lemma:v}
    Under Assumptions~\ref{as:act_mapping} and \ref{as:state_mapping}, for policy $\pi$ in the original problem domain, there exists $\pi_i$ such that
    \begin{align*}
        v^\pi (s) & = \sum_{a \in A} \int_{\{e\} = f^{-1}(a)} \pi_i(e | x=\phi(s)) Q^{\pi}(s, a) \,de \,.
    \end{align*}
\end{lemma}
\begin{proofsketch}
    Based on the Bellman equation for the value function $v^{\pi}$ in the original domain, we introduce the embedded state $x=\phi(s)$ using Assumption \ref{as:state_mapping}.
    By the law of total probability and Assumption \ref{as:act_mapping}, we then introduce the embedded action $e$.
    From the Markovian property and the definition of our model, we can derive Claims \ref{claim:a_s_cond_indep_x} - \ref{claim:drop_e_x_for_a} on (conditional) independence in Appendix \ref{appendix:claims}, which then allow us to remove the original state and action from the expression of the policy, leaving us with $\pi_i$.
    See Appendix \ref{appendix:prrof_L1} for the complete proof.
\end{proofsketch}

By Assumptions \ref{as:act_mapping} and \ref{as:state_mapping}, Lemma \ref{lemma:v} allows us to express the overall policy $\pi_o$ in terms of the internal policy $\pi_i$ as
\begin{align}
    \pi_o(a | s) & = \int_{\{e\} = f^{-1}(a)} \pi_i(e | x=\phi(s)) \,de \,,
\end{align}
bridging the gap between the original problem domain and the policy in embedding spaces $\mathcal{X}$ and $\mathcal{E}$.
Under $\pi^*$ in the original domain, we define $v^*:=v^{\pi^*}$ and $Q^*:=Q^{\pi^*}$ for the ease of discussions.
Then using Lemma \ref{lemma:v}, we now prove the existence of an overall policy $\pi_o$ that is optimal.
\begin{theorem}
    \label{theorem:pi}
    Under Assumptions \ref{as:act_mapping} and \ref{as:state_mapping}, there exists an overall policy $\pi_o$ that is optimal, such that $v^{\pi_o} = v^*$.
\end{theorem}
\begin{proof}
    Under finite state and action sets, bounded rewards, and $\gamma \in [0, 1)$, at least one optimal policy $\pi^*$ exists.
    From Lemma \ref{lemma:v}, we then have
    \begin{align}
        \label{eq:optimal_policy_pi_o}
        v^* (s) & = \sum_{a \in A} \int_{\{e\} = f^{-1}(a)} \pi_i(e | \phi(s)) Q^*(s, a) \,de \,.
    \end{align}
    Thus, $\exists$ $\phi$, $f$, and $\pi_i$, representing an overall policy $\pi_o$, which is optimal, i.e., $v^{\pi_o} = v^{*}$.
\end{proof}


Theorem~\ref{theorem:pi} suggests that in order to get the optimal policy $\pi^*$ in the original domain, we can focus on the optimization of the overall policy $\pi_o$, which is discussed in the next section.

\subsection{Architecture Embodiment and Training}
In this section, we first present the implementation and the training of the state-action embedding model (illustrated in Figure~\ref{fig:embed_model}). Based on this method and according to Theorem~\ref{theorem:pi}, we then propose a strategy to train the overall policy $\pi_o$, where functions $\phi$ and $f$ are iteratively updated.

\subsubsection{Joint Training of the State-Action Embedding}
\label{method:embedding}
\begin{figure*}[htbp]
    \centering
    \begin{subfigure}[t]{.22\textwidth}
        \centering
        \includegraphics[width=\linewidth]{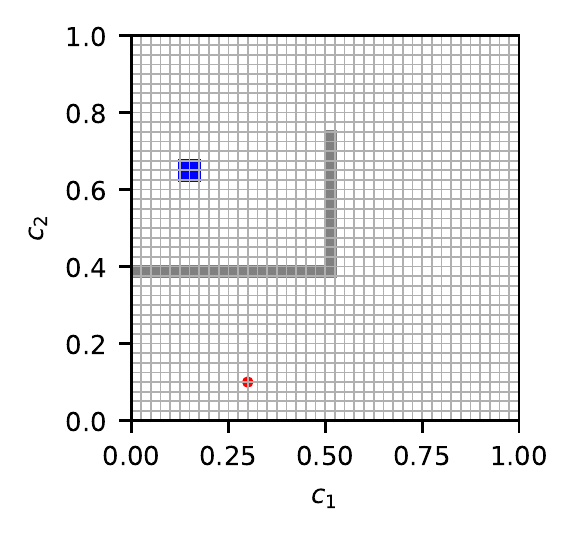}
        \vspace{-1.9em}
        \caption{Discrete gridworld}
        \label{fig:states_grid}
    \end{subfigure}\hspace*{\fill}
    \begin{subfigure}[t]{.22\textwidth}
        \centering
        \includegraphics[width=\linewidth]{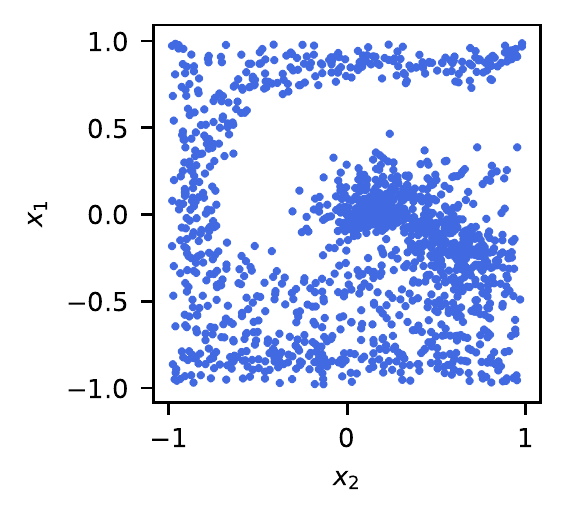}
        \vspace{-1.9em}
        \caption{State embeddings}
        \label{fig:state_embeddings}
    \end{subfigure}\hspace*{\fill}
    \begin{subfigure}[t]{.22\textwidth}
        \centering
        \includegraphics[width=\linewidth]{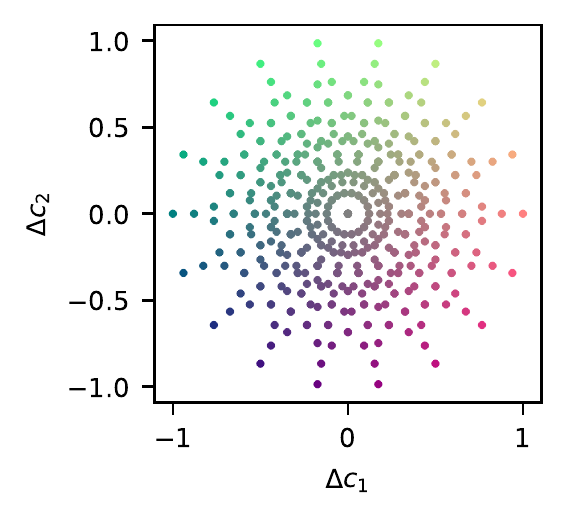}
        \vspace{-1.9em}
        \caption{Action movements}
        \label{fig:action_movements}
    \end{subfigure}\hspace*{\fill}
    \begin{subfigure}[t]{.22\textwidth}
        \centering
        \includegraphics[width=\linewidth]{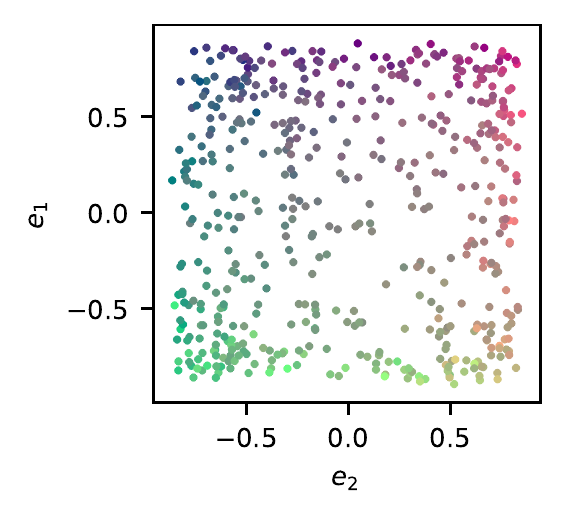}
        \vspace{-1.9em}
        \caption{Action embeddings}
        \label{fig:action_embeddings}
    \end{subfigure}\hspace*{\fill}
    \vspace{-0.3em}
    \caption{State and action embeddings for a discrete state gridworld with $1,600$ states and $2^9$ actions (9 actuators).
        The embedding dimension for both states and actions is set to $2$ for visualization purposes.
        (a) Original environment, i.e., a discrete gridworld as described in Section \ref{exp:poc}.
        The blue area is the goal state and the red dot is the starting point of the agent.
        (b) Learned state embeddings in continuous state embedding space.
        (c) Displacement in Cartesian co-ordinates ($c_1$ and $c_2$) caused by actions.
        Each action is colored according to this displacement with $[R = \Delta c_1, G = \Delta c_2 , B = 0.5]$.
        (d) Learned action embeddings in continuous action embedding space. Both state and action embeddings clearly capture the structure of the environment.}
    \label{fig:embedding_vis}
    \vspace{-0.3em}
\end{figure*}
In this section, we elaborate on our proposed embedding model for jointly learning state and action embeddings.
Specifically, we require two components: (i) function $\phi$ that projects $S_t$ into embedding space $\mathcal{X}$ and (ii) function $f$ that maps each point in embedding space $\mathcal{E}$ to an action $A_t$ in the original problem domain.
Since these functions are not known in advance, we train estimators using a model of the environment (see Figure~\ref{fig:embed_model}).
This environment model requires two further components: (iii) function $g: \mathcal{A} \mapsto \mathcal{E}$ and (iv) transition model $T: \mathcal{E} \times \mathcal{X} \mapsto \mathcal{S}$ that predicts the next state from the concatenated embeddings of the current state and action.
Note that $g$ is a \emph{one-to-one} mapping, since by Assumption~\ref{as:act_mapping}, no two actions have exactly the same embedding.
Denote the estimators of components (i)-(iv) by $\hat{\phi} \,, \hat{f} \,, \hat{g} \,, \hat{T}$; the target of the environment model is
\begin{align}
    \begin{aligned}
        \hat{P}(S_{t+1} | S_t, A_t) = \hat{T}(S_{t+1} | X_t, E_t) \hat{g}(E_t | A_t) \hat{\phi}(X_t | S_t) \,.
    \end{aligned}
\end{align}
The difference between the true transition probabilities $P(S_{t+1} | S_t, A_t)$ and the estimated probabilities $\hat{P}(S_{t+1} | S_t, A_t)$ can be measured using the Kullback-Leibler (KL) divergence, where the expectation is over the true distribution $P(S_{t+1} | S_t, A_t)$, i.e.,
\begin{align}
    \begin{aligned}
         D_{KL} (P || \hat{P}) = - \mathbb{E}_{S_{t+1} \sim P(S_{t+1} | S_t, A_t)} \Big [\ln \Big ( \frac{\hat{P}(S_{t+1} | S_t, A_t)}{P(S_{t+1} | S_t, A_t)} \Big ) \Big ] \,.
    \end{aligned}
    \label{eq:DKL}
\end{align}
From tuples $(S_{t+1}, S_t, A_t)$, we can compute a sample estimate of this using Equation~(\ref{eq:DKL}).
In Equation~(\ref{eq:DKL}), the denominator does not depend on $\hat{\phi}$, $\hat{g}$, or $\hat{T}$. Therefore, we define the loss function for the embedding model as
\begin{align}
    \begin{aligned}
         \mathcal{L}(\hat{\phi}, \hat{g}, \hat{T}) = - \mathbb{E}_{S_{t+1} \sim P(S_{t+1} | S_t, A_t)} \Big [\ln (\hat{P}(S_{t+1} | S_t, A_t) ) \Big ] \,.
    \end{aligned}
    \label{eq:loss_embedding}
\end{align}
Note that the estimator $\hat{f}$ is not directly included in the embedding model.
Instead, we find $\hat{f}$ by minimizing the error between the original action and the action reconstructed from the embedding.
We can define this as
\begin{align}
    \label{eq:loss_f}
    \mathcal{L}(\hat{f}) = - \ln (\hat{f}(A_t | E_t) ) \,.
\end{align}
Then, all components of our embedding model can be learned by minimizing the loss functions $\mathcal{L}(\hat{\phi}, \hat{g}, \hat{T})$ and $\mathcal{L}(\hat{f})$.
Note that the embedding model is trained in two steps,
firstly by minimizing the loss in Equation (\ref{eq:loss_embedding}) to update $\hat{\phi}$, $\hat{g}$, and $\hat{T}$ and secondly by minimizing the loss in Equation (\ref{eq:loss_f}) to update $\hat{f}$.
Here, we give the target of the embedding model for the case of discrete state and action spaces.
However, the model is equally applicable to continuous domains.
In this case, the loss functions $\mathcal{L}(\hat{\phi}, \hat{g}, \hat{T})$ and $\mathcal{L}(\hat{f})$ are replaced by loss functions suitable for continuous domains.
For continuous domains in Section~\ref{sed:experiments}, we adopt a mean squared error loss instead of the losses given in Equations~(\ref{eq:loss_embedding}-\ref{eq:loss_f}), but other loss functions may also be applicable.
For experiments, we parameterize all components of the embedding model illustrated in Figure \ref{fig:embed_model} as neural networks.
More details can be found in Appendix \ref{parameters:model_imp}.

\subsubsection{State-Action Embedding and Policy Learning}
\label{method:algorithm}

Theorem~\ref{theorem:pi} only shows that optimizing $\pi_o$ can help us achieve optimality in the original domain. Then a natural question is whether we can optimize $\pi_o$ by directly optimizing the internal policy $\pi_i$ using the state/action embeddings in Section~\ref{method:embedding}. To answer this question, we first mathematically derive that updating $\pi_i$ is equivalent to updating $\pi_o$. 
We then proceed to present an iterative algorithm for learning the state-action embeddings and the internal policy.


Suppose $\pi_i$ is parameterized by $\theta$. Then with the objective of optimizing $\pi_o$ by only updating $\pi_i$,
we define the performance function of $\pi_o$ as\footnote{Equation (\ref{eq:perf_overall}) is for the overall policy with discrete states and actions. Note that our approach is also applicable in continuous domains. $\sum_{s \in \mathcal{S}}$ and $\sum_{a \in \mathcal{A}}$ would then be replaced by $\int_s$ and $\int_a$, respectively.\looseness=-1}
\begin{align}
    \label{eq:perf_overall}
    J_o(\phi, \theta, f) = \sum_{s \in \mathcal{S}} d_0(s) \sum_{a \in \mathcal{A}} \pi_o(a | s) Q^{\pi_o}(s, a) \,.
\end{align}
Let the state-action-value function for the internal policy be $Q^{\pi_i} (x, e) = \mathbb{E}[\sum_{k=0}^{\infty} \gamma^k R_{t+k} | \pi_i, X_t = x, E_t = e]$.
We can then define the performance function of the internal policy as
\begin{align}
    J_i(\theta) = \int_{x\in\mathcal{X}} d_0(x) \int_{e\in\mathcal{E}} \pi_i(e | x) Q^{\pi_i}(x, e) \, de \, dx \,.
\end{align}
The parameters of the overall policy $\pi_o$ can then be learned by updating its parameters $\theta$ in the direction of $\partial J_o(\phi, \theta, f) / \partial \theta$, while the parameters of the internal policy can be learned by updating $\theta$ in the direction of $\partial J_i(\theta) / \partial \theta$.
With the following assumption, we then have Lemma \ref{lemma:PG}.
\begin{assumption}
    \label{as:unique_s_to_x}
    The state embedding function $\phi(s)$ maps each state $s$ to a unique state embedding $x = \phi(s)$, i.e.,
    $
        \forall s_i \neq s_j \,, P(\phi(s_i) = \phi(s_j)) = 0
    $.
\end{assumption}
Note that Assumption \ref{as:unique_s_to_x}, defining $\phi$ as a \emph{one-to-one} mapping, is not theoretically guaranteed by the definition of our embedding model.
Nevertheless, we test this empirically in Section~\ref{sed:experiments} and find that no two states share the same embedded representation in any of our experiments, thus justifying Assumption \ref{as:unique_s_to_x} in practical scenarios considered in this work.
\begin{lemma}
    \label{lemma:PG}
    Under Assumptions \ref{as:act_mapping}--\ref{as:unique_s_to_x}, for all deterministic functions $f$ and $\phi$, which map each point $e \in \mathcal{E}$ and $s \in \mathcal{S}$ to an action $a \in \mathcal{A}$ and to an embedded state $x \in \mathcal{X}$, and the internal policy $\pi_i$ parameterized by $\theta$, the gradient of the internal policy's performance function $\frac{\partial J_i(\theta)}{\partial \theta}$ equals the gradient of the overall policy's performance function $\frac{\partial J_o(\phi, \theta, f)}{\partial \theta}$, i.e.,
    \begin{align*}
        \frac{\partial J_i(\theta)}{\partial \theta} & = \frac{\partial J_o(\phi, \theta, f)}{\partial \theta}  \,.
    \end{align*}
\end{lemma}
\begin{proofsketch}
    Assumptions~\ref{as:act_mapping} and \ref{as:unique_s_to_x} allow us to show the equivalence between the internal state-action-value function and the overall state-action-value function. Using this equivalence, we can then remove $\phi$ and $f$ from $\partial J_o(\phi, \theta, f)/ \partial \theta$ and are left with the gradient of the internal policy $\pi_i$ only.
    The complete proof is deferred to Appendix \ref{appendix:prrof_L2}.
\end{proofsketch}

Lemma \ref{lemma:PG} shows that the updates to the internal policy $\pi_i$ and the overall policy $\pi_o$ are equivalent.
This allows us to optimize the overall policy by making updates directly to the internal policy $\pi_i$, thereby avoiding the potentially intractable computation of the inverse functions $f^{-1}$ and $\phi^{-1}$.
Since there are no special restrictions on the internal policy $\pi_i$, we can use any policy gradient algorithm designed for continuous control to optimize the policy.
In addition, we can also iteratively update parameters for jointly learning $\phi$, $f$, and $\pi_i$, as shown in Algorithm \ref{algo:training}.
\begin{algorithm}[htb]
    \caption{Joint Training of State-Action Embeddings and the Internal Policy}
    \label{algo:training}
    \begin{algorithmic}
        \STATE Initialize the state and action embeddings (optional pre-training)\;
        \FOR{Epoch = 0, 1, 2, \ldots}
        \FOR{t = 1, 2, 3, \ldots}
        \STATE Generate state embedding $X_t = \phi(S_t)$ \;
        \STATE Sample action embedding $E_t \sim \pi_i(\cdot | X_t)$\;
        \STATE Map embedded action to $A_t = f(E_t)$ \;
        \STATE Execute $A_t$ in the environment to observe $S_t, R_t$\;
        \STATE Update the $\pi_i$ and the critic using some policy gradient algorithms\;
        \STATE Update $\phi$, $g$, $T$, and $f$ by minimizing the losses in Equations~(\ref{eq:loss_embedding}) and (\ref{eq:loss_f})\;
        \ENDFOR
        \ENDFOR

    \end{algorithmic}
\end{algorithm}

\section{Empirical Evaluation}
\label{sed:experiments}
We evaluate our proposed architecture, denoted by JSAE (Jointly-trained State Action Embedding), on game-based applications, robotic control, and a real-world recommender system, covering different combinations of discrete and continuous state and action spaces.
Our methodology is evaluated using Vanilla Policy Gradient (VPG) \citep{SpinningUp2018}, Proximal Policy Optimization (PPO) \citep{schulman2017PPO}, and Soft Actor Critic (SAC) \citep{tuomas_SAC_2018} algorithms.
We benchmark the performance against these algorithms without embeddings and with action embeddings generated by the method called policy gradients with Representations for Actions (RA) proposed by \citet{Chandak2019}.
To isolate the effect of dimensionality reduction, we additionally benchmark against the aforementioned algorithms with embeddings generated from two auto-encoders (AE)~--~one for states and one for actions.
The used auto-encoders consist of a simple two-layer feed-forward network, where the weights of the first layer after training are used as embeddings.
We pre-train the embedding models on randomly collected samples for all experiments and enable continuous updates for the Ant-v2 and recommender system environments.
\begin{figure*}[htbp]
    \centering
    \begin{subfigure}[t]{.34\textwidth}
        \centering
        \includegraphics[width=\linewidth]{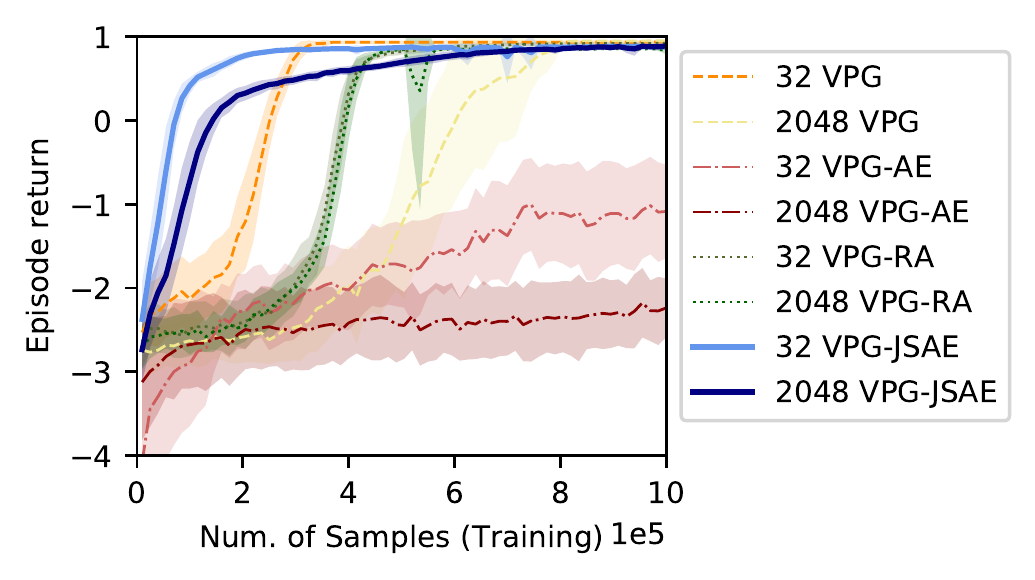}
        \vspace{-1.7em}
        \caption{Gridworld with discrete states and actions for different numbers of actions}
        \label{fig:res_grid_disc}
    \end{subfigure}\hspace*{\fill}
    \begin{subfigure}[t]{.34\textwidth}
        \centering
        \includegraphics[width=\linewidth]{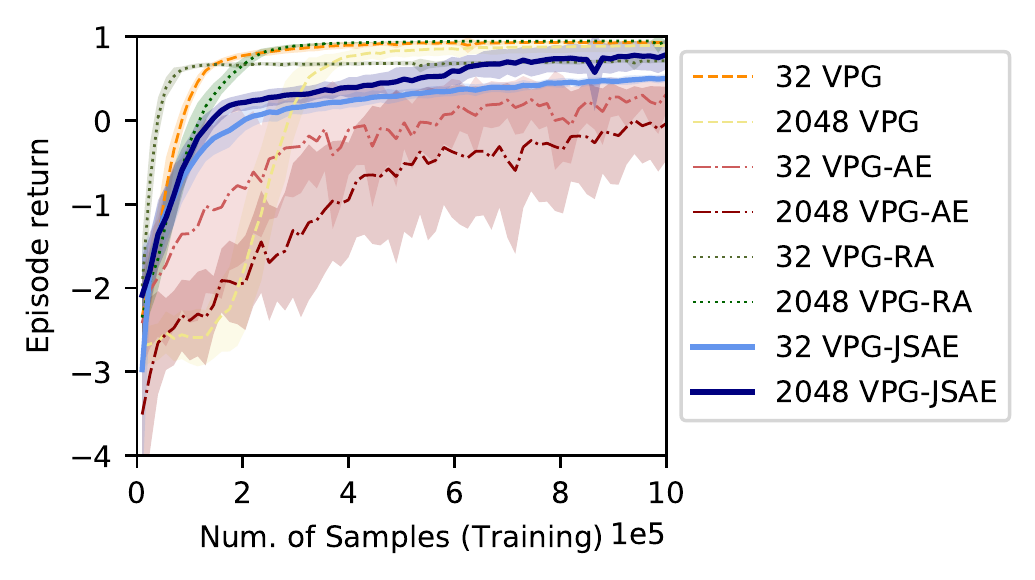}
        \vspace{-1.7em}
        \caption{Gridworld with continuous states and discrete actions for different numbers of actions}
        \label{fig:res_grid_cont}
    \end{subfigure}\hspace*{\fill}
    \begin{subfigure}[t]{.32\textwidth}
        \centering
        \includegraphics[width=\linewidth]{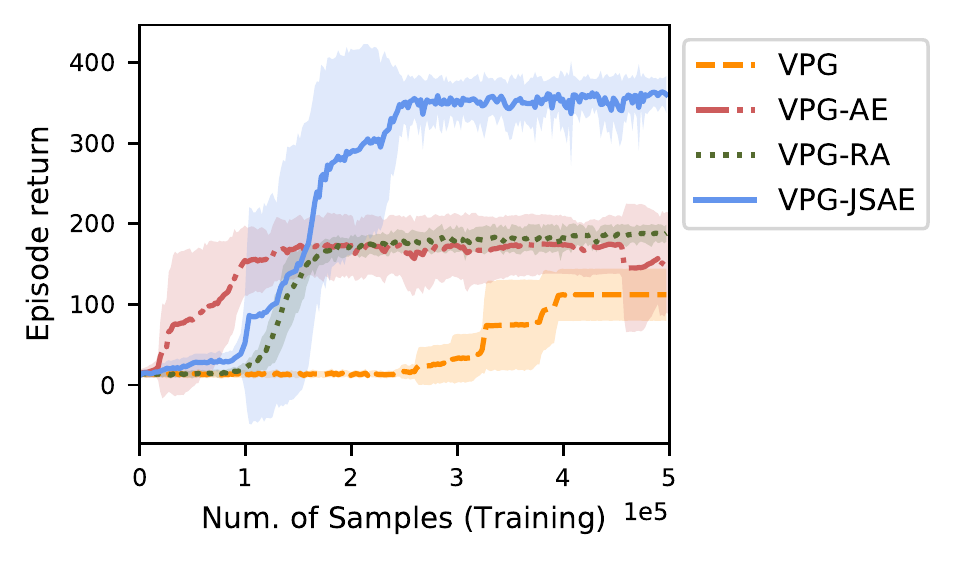}
        \vspace{-1.7em}
        \caption{Slotmachine}
        \label{fig:res_slotmachine}
    \end{subfigure}
    \begin{subfigure}[t]{.34\textwidth}
        \centering
        \includegraphics[width=\linewidth]{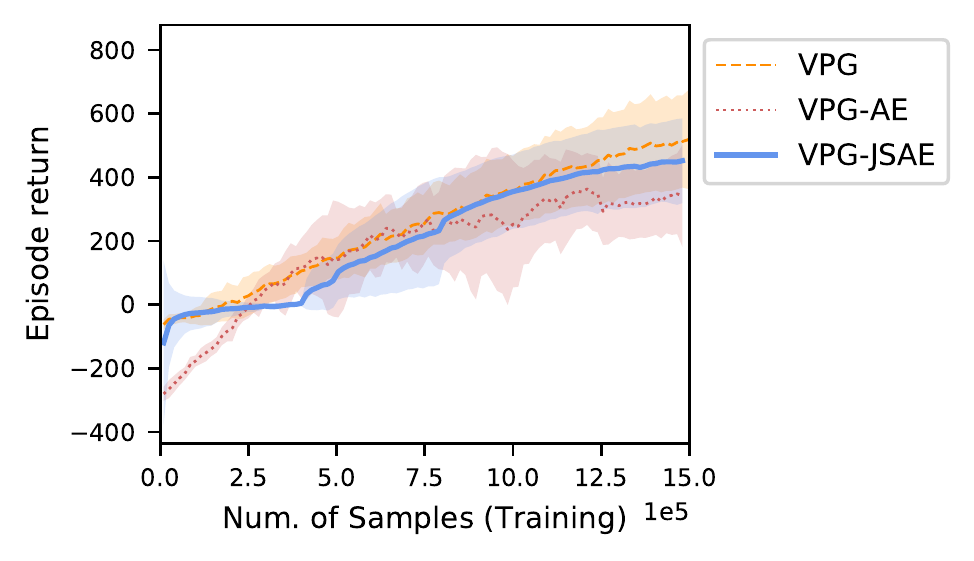}
        \vspace{-1.7em}
        \caption{Half-Cheetah}
        \label{fig:res_mujoco}
    \end{subfigure}\hspace*{\fill}
    \begin{subfigure}[t]{.32\textwidth}
        \centering
        \includegraphics[width=\linewidth]{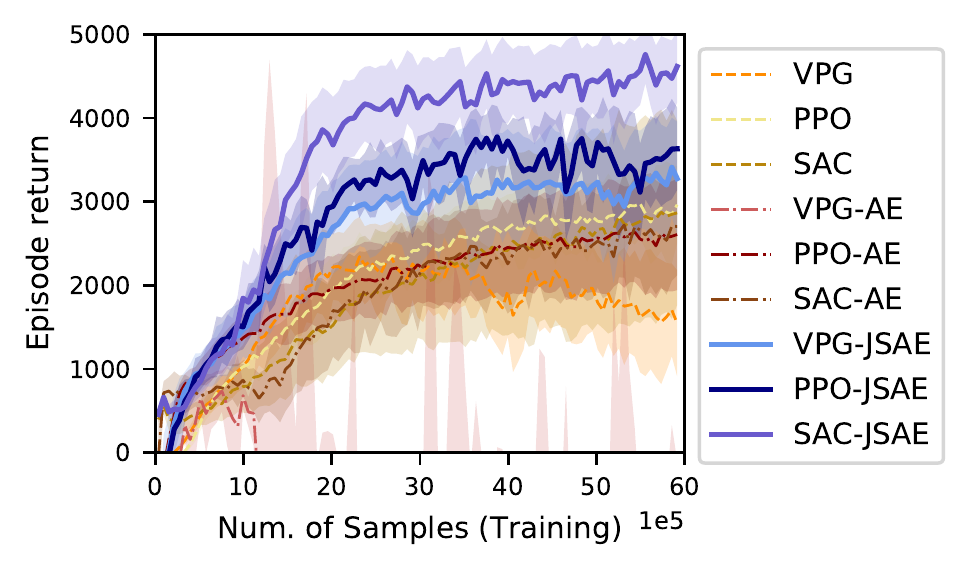}
        \vspace{-1.7em}
        \caption{Ant-v2}
        \label{fig:res_ant}
    \end{subfigure}\hspace*{\fill}
    \begin{subfigure}[t]{.33\textwidth}
        \centering
        \includegraphics[width=\linewidth]{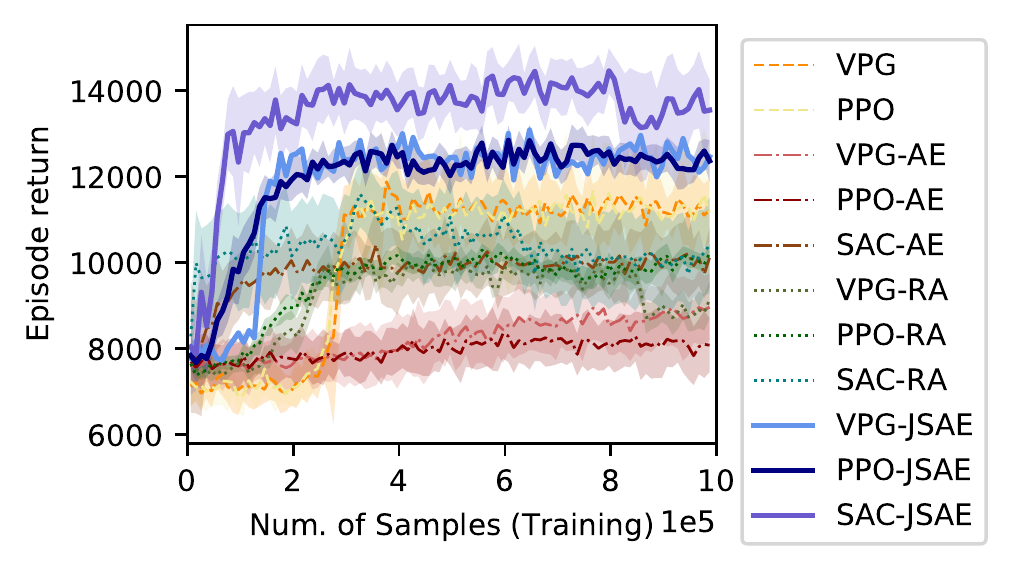}
        \vspace{-1.7em}
        \caption{Recommender system}
        \label{fig:res_recommender}
    \end{subfigure}
    \caption{Performance of our approach (JSAE, Jointly-trained State Action Embedding) compared against benchmarks without embeddings and/or benchmarks with action embeddings by RA (all results are the average return over $10$ episodes, with mean and standard deviation over $10$ random seeds).}
    \label{fig:experiment_results}
    \vspace{-0.4em}
\end{figure*}

\subsection{Proof-of-Concept: Gridworld and Slotmachine}
\label{exp:poc}

\emph{Gridworld:} It is similar to that used by \citet{Chandak2019}.
States are given as a continuous coordinate, while actions are defined via $n$ actuators, equally spaced around the agent, which move the agent in the direction they are pointing towards.
Then each combination of actuators forms an action, resulting in $2^n$ unique actions.
We run two sets of experiments in this environment: (i) we use the continuous state directly and (ii) we discretize the coordinate in a $40\times40$ grid.
The agent receives small step and collision penalties and a large positive reward for reaching the goal state.
Figure \ref{fig:embedding_vis} illustrates the learned embeddings for this environment.
The obtained embeddings for states and actions illustrate that our approach is indeed able to obtain meaningful representations.
Actions are embedded according to the displacement in the Cartesian coordinates they represent (Figure \ref{fig:action_embeddings}).
Figure \ref{fig:state_embeddings} shows that states are embedded according to the coordinate they represent and interestingly, the embeddings capture the $L$-shaped obstacle in the original problem domain.\looseness=-1

\emph{Slotmachine: }
It consists of four reels with $6$ values per reel, each of which can be individually spun by the agent for a fraction of a full turn.
Then, each unique combination of these fractions constitutes an action and the agent receives rewards for aligning reels.
The discrete state is given by the numbers on the front of the reels.

Both scenarios yield large discrete action spaces, rendering them well-suited for evaluating our approach.
The results for the gridworld and slotmachine environments are shown in Figure \ref{fig:res_grid_disc}, \ref{fig:res_grid_cont}, and \ref{fig:res_slotmachine}.
Our approach (JSAE) outperforms all three benchmarks on the discretized gridworld and the slotmachine environments and is comparable for the continuous-state gridworld.
Such observation suggests that our approach is particularly well suited for discrete domains.
Intuitively, this is because the relationship between different states or actions in discrete domains is usually not apparent, e.g., in one-hot encoding.
In continuous domains, however, structure in the state and action space is already captured to some extent, rendering it harder to uncover additional structure by embeddings; nevertheless, when a continuous problem becomes more complicated, our embedding method again shows effectiveness in capturing useful information (see Section~\ref{exp:robotic}).
Notably, the auto-encoder generated embeddings lead to worse performance on both gridworld domains and only slightly improve convergence speed on the slotmachine environment.
As discrete state and action spaces do not necessarily contain an inherent structure, the auto-encoder can only reduce the dimensionality without capturing the underlying structure of the problem domain, and potentially embed behaviorally different states/actions close to each other.
Such initial findings suggest that it is not the dimensionality reduction, but our approach's ability to capture the environment structure in the embedding space that improves the performance.

\subsection{Robotic Control}
\label{exp:robotic}
We use the \textit{half-cheetah-v2} and the \textit{Ant-v2} environments from the robotic control simulator Mujoco \citep{mujoco_paper} and limit the episode length to $250$ and $1000$ steps respectively.
Here, the agent observes continuous states (describing the current position of the cheetah or ant respectively) and then decides on the force  to apply to each of the joints of the robot (i.e., continuous actions).
Rewards are calculated based on the distance covered in an episode.
On these domains, the auto-encoder baseline performs comparable to the baselines without embeddings (Figure \ref{fig:res_mujoco}, \ref{fig:res_ant}).
While the auto-encoder is now able to preserve much of the inherent structure of the state-action space, as we have continuous state and action spaces, it is not able to use the interaction between states and actions to structure the embedding space.
From the results presented in Figure \ref{fig:res_mujoco}, we observe that our method (JSAE) does not outperform the benchmarks on the half-cheetah-v2 environment.
On the Ant-v2 environment, however, our approach outperforms the benchmarks.
Similar to the experiments in a continuous gridworld, it is less likely that embedding uncovers additional structure on top of the structure inherent to the continuous state and action, explaining the comparable performance with and without embeddings on the half-cheetah-v2 domain.
Nevertheless, in the more challenging Ant-v2 environment, the state and action representation is more complex and the agent might benefit from additional structure uncovered via our approach.
By contrasting our approach with the auto-encoder baseline, we can again confirm that the improved performance does not stem from a mere dimensionality reduction in the state and action representations.\looseness=-1

\subsection{Recommender System}
\label{exp:recommender}
In addition, we also test our methodology on a real-world application of a recommender system.
We use data on user behavior from an e-commerce store collected in 2019 \citep{recommender_data}.
The environment is constructed as an $n$-gram based model, following \citet{shani_MDP_recommender_2005}.
Based on the $n$ previously purchased items, a user's purchasing likelihood is computed for each item in the store, leading to a stochastic transition function.
Recommending an item then scales the likelihood for that item by a pre-defined factor.
States are represented as a concatenation of the last $n$ purchased items and each item forms an action.
In this environment we have $835$ items (actions) and approx. $700,000$ states.
The results obtained under various RL methods are reported in Figure~\ref{fig:res_recommender}.
Note that we do not run an SAC benchmark, as this algorithm is designed for continuous action spaces only.
\begin{figure}[b]
    \vspace{-1.6em}
    \centering
    \includegraphics[width=.7\linewidth]{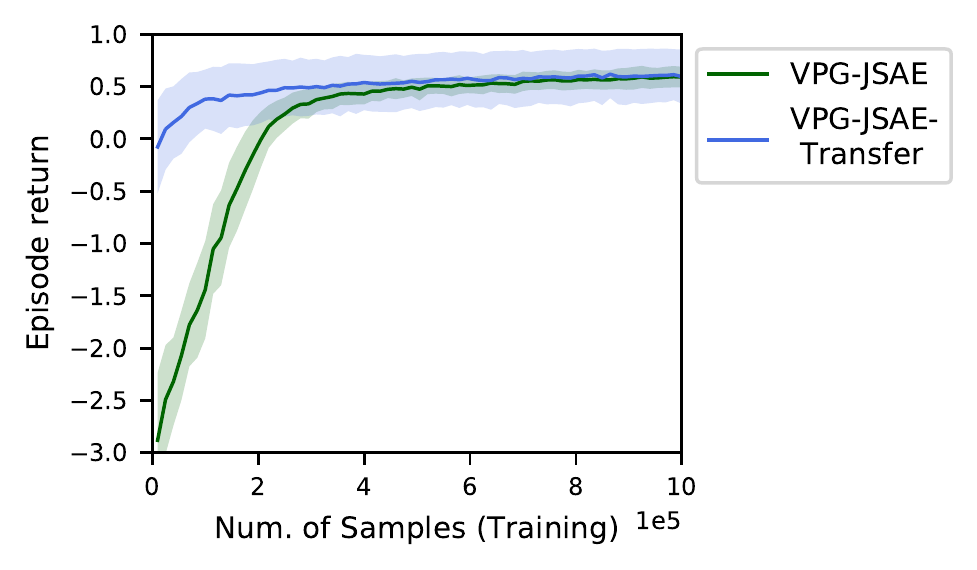}
    \vspace{-1.5em}
    \caption{Transfer learning in the discrete-state gridworld.}
    \label{fig:res_transfer_learning}
\end{figure}
We find that our approach leads to significant improvements in both the convergence speed and final performance compared to the benchmarks.
This result confirms that our method is particularly useful in the presence of large discrete state and action spaces.
Interestingly, the PPO and VPG benchmarks outperform the benchmarks using the RA and AE methodologies.
We conjecture that the action embeddings generated in RA and AE on this environment in fact obfuscate the effect an action has, and thus limit the agent's performance.\looseness=-1

\subsection{Application in Transfer Learning}
\label{exp:transfer}
In addition to improving the convergence speed and the accumulated reward, we test whether the policies learned in one environment can be transferred to a similar one.
We consider two environments (old and new) that differ by only the state space or the action space.
To leverage the previous experience, if the two environments have the same state (or action) space, the state (or action) embedding component in the new environment is initialized with the weights learned from the old environment; we then train our model in this new environment.
For evaluation, we use two gridworlds, where only the numbers of actions differ, i.e., 512 (old) and 2048 (new).
Results in Figure \ref{fig:res_transfer_learning} show that the model with the transferred knowledge outperforms the case where the entire model is trained from scratch, in terms of the convergence speed.
Thus, it shows that the previously learned policy serves as a critical initialization for the new environment and sheds light on the potential of our approach in transfer learning; further evaluations are left for future work.\looseness=-1

\section{Conclusion}
\label{sec:conclusion}
In this paper, we presented a new architecture for jointly training state/action embeddings and combined this with common reinforcement learning algorithms. Our theoretical results confirm the validity of the proposed approach, i.e., the existence of an optimal policy in the embedding space.
We empirically evaluated our method on several environments, where it outperforms state-of-the-art RL approaches in complex large-scale problems.
Our approach is easily extensible as it can be combined with most existing RL algorithms and there are no special restrictions on  embedding model parameterizations.\looseness=-1


\appendix
\section{Appendix}
\label{sec:appendix}
\subsection{Claims}
\label{appendix:claims}

In addition to Assumptions \ref{as:act_mapping} and \ref{as:state_mapping}, we derive three claims on conditional independence from the definition of our embedding model and the Markovian property of the environment, listed below.

\begin{claim}
    \label{claim:a_s_cond_indep_x}
    Action $A_t$ is conditionally independent of $S_t$ given a state embedding $X_t$:
    \begin{align*}
        \begin{aligned}
             P(A_t = a, S_t = s | X_t = x) = P(A_t = a | X_t = x)P(S_t = s | X_t = x) \,.
        \end{aligned}
    \end{align*}
\end{claim}

\begin{claim}
    \label{claim:a_s_cond_indep_x_e}
    Action $A_t$ is conditionally independent of $S_t$ given state and action embeddings $X_t$ and $E_t$:
    \begin{align*}
        \begin{aligned}
             & P(A_t = a, S_t = s | X_t = x, E_t = e) =                       \\
             & P(A_t = a | X_t = x, E_t = e)P(S_t = s | X_t = x, E_t = e) \,.
        \end{aligned}
    \end{align*}
\end{claim}

\begin{claim}
    \label{claim:s_n_cond_indep_a_s}
    Next state $S_{t+1}$ is conditionally independent of $E_t$ and $X_t$ given $A_t$ and $S_t$:
    \begin{align*}
        P(S_{t+1} = s', E_t = e, X_t = x | A_t = a, S_t = s) = \\
        P(S_{t+1} = s' | A_t = a, S_t = s) P(E_t = e, X_t = x | A_t = a, S_t = s) \,.
    \end{align*}
\end{claim}

Using Claims \ref{claim:a_s_cond_indep_x} - \ref{claim:s_n_cond_indep_a_s}, we can derive two further auxiliary claims that will be used in the proof of Lemma \ref{lemma:v}.

\begin{claim}
    \label{claim:drop_e_x_for_s_n}
    The probability of next state $S_{t+1}$ is independent of $E_t$ and $X_t$, given state $S_t$ and action $A_t$, i.e.,
    \begin{align*}
        \begin{aligned}
             & P(S_{t+1} = s' | E_t = e, X_t = x, A_t = a, S_t = s) = \\
             & P(S_{t+1} = s' | A_t = a, S_t = s) \,.
        \end{aligned}
    \end{align*}
    \begin{proof}
        From Claim \ref{claim:s_n_cond_indep_a_s}, we have
        \begin{align*}
            P(s', e, x | a, s) =                & P(s' | a, s)P(e, x | a, s)  \\
            P(s' | e, x, a, s) P(e, x | a, s) = & P(s' | a, s) P(e, x | a, s) \\
            P(s' | e, x, a, s) =                & P(s' | a, s) \,.
        \end{align*}
    \end{proof}
\end{claim}

\begin{claim}
    \label{claim:drop_e_x_for_a}
    The probability of action $A_t$ is independent of $S_t$ and $X_t$, given action embedding $E_t$, i.e.,
    \begin{align*}
    \begin{aligned}
        P(A_t = a | S_t = s, X_t = x, E_t = e) = P(A_t = s | E_t = e) \,.
    \end{aligned}
    \end{align*}
    \begin{proof}
        From Claim \ref{claim:a_s_cond_indep_x_e}, we have
        \begin{align*}
            P(a, s | x, e) = & P(a | x, e)P(s | x, e) \\
            P(a | s, x, e) = & P(a | x, e) \,.
        \end{align*}
        Since action $a$ only depends on the action embedding $e$ in our model, this becomes
        \begin{align*}
            P(a | s, x, e) = & P(a | e) \,. 
        \end{align*}
    \end{proof}
\end{claim}

\subsection{Proof of Lemma~\ref{lemma:v}}
\label{appendix:prrof_L1}
\textbf{Lemma 1.}
\textit{
    Under Assumptions \ref{as:act_mapping} and \ref{as:state_mapping}, for policy $\pi$ in the original problem domain, there exists $\pi_i$ such that}
\begin{align*}
    v^\pi (s) & = \sum_{a \in \mathcal{A}} \int_{\{e\} = f^{-1}(a)} \pi_i(e | x=\phi(s)) Q^{\pi}(s, a) \,de
\end{align*}
\begin{proof}
    The Bellman equation for a MDP is given by
    \begin{equation*}
        v^\pi (s) = \sum_{a \in \mathcal{A}} \pi(a|s) \sum_{s' \in \mathcal{S}} P(s' | s, a) G \,,
    \end{equation*}
    where $G$ denotes the return, i.e. $[\mathcal{R}(s, a) + \gamma v^{\pi}(s')]$, which is a function of $s, a$, and $s'$.
    By re-arranging terms we get
    \begin{align*}
        v^\pi (s) & = \sum_{a \in \mathcal{A}} \sum_{s' \in \mathcal{S}} \pi(a|s) P(s' | s, a) G                       \\
                  & = \sum_{a \in \mathcal{A}} \sum_{s' \in \mathcal{S}} \pi(a|s) \frac{P(s' , s, a)}{P(s, a)} G       \\
                  & = \sum_{a \in \mathcal{A}} \sum_{s' \in \mathcal{S}} \pi(a|s) \frac{P(s' , s, a)}{\pi(a|s) P(s)} G \\
                  & = \sum_{a \in \mathcal{A}} \sum_{s' \in \mathcal{S}} \frac{P(a | s' , s) P( s', s)}{P(s)} G \,.    
    \end{align*}
    Since $s$ can be deterministically mapped to $x$ via $x = \phi(s)$, by Assumption \ref{as:state_mapping}, we have
    \begin{align*}
        v^\pi (s) & = \sum_{a \in \mathcal{A}} \sum_{s' \in \mathcal{S}} \frac{P(x, a | s' , s) P( s', s)}{P(s)} G               \\
                  & = \sum_{a \in \mathcal{A}} \sum_{s' \in \mathcal{S}} \frac{P(x, a | s' , s) P( s', s) P(x|s)}{P(x|s) P(s)} G \\
                  & = \sum_{a \in \mathcal{A}} \sum_{s' \in \mathcal{S}} \frac{P(x, a , s' , s) P(x|s)}{P(x, s)} G               \\
                  & = \sum_{a \in \mathcal{A}} P(x|s) \sum_{s' \in \mathcal{S}} \frac{P(x, a , s' , s)}{P(x, s)} G               \\
                  & = \sum_{a \in \mathcal{A}} P(x|s) \sum_{s' \in \mathcal{S}} \frac{P(a , s' | x, s)P(x, s)}{P(x, s)} G        \\
                  & = \sum_{a \in \mathcal{A}} P(x|s) \sum_{s' \in \mathcal{S}} P(a , s' | x, s) G                               \\
                  & = \sum_{a \in \mathcal{A}} P(x|s) \sum_{s' \in \mathcal{S}} P(s' | a, x, s) P(a | s, x) G \,.                
    \end{align*}
    From Claim \ref{claim:a_s_cond_indep_x}, we know that $P(a | s, x) = P(a | x)$. Therefore,
    \begin{align*}
        v^\pi (s) & = \sum_{a \in \mathcal{A}} P(x|s) \sum_{s' \in \mathcal{S}} P(s' | a, x, s) P(a | x) G \,. 
    \end{align*}
    Since $P(x | s)$ is deterministic by Assumption \ref{as:state_mapping} and evaluates to $1$ for the representation of $\phi(s) = x$, we can rewrite the equation above to
    \begin{align*}
        v^\pi (s) & = \sum_{a \in \mathcal{A}} \sum_{s' \in \mathcal{S}} P(s' | a, x, s) P(a | x) G \,. 
    \end{align*}

    We now proceed to establish the relationship with the action embedding.
    From above we have
    \begin{align*}
        v^\pi (s) & = \sum_{a \in \mathcal{A}} \sum_{s' \in \mathcal{S}} P(a | x) P(s' | a, x, s) G                        \\
                  & = \sum_{a \in \mathcal{A}} \sum_{s' \in \mathcal{S}}  P(a| x) \frac{P(s', a, x, s)}{P(a, x, s)}  G \,. 
    \end{align*}
    From Claim \ref{claim:a_s_cond_indep_x}, $a$ and $s$ are conditionally independent given $x$.
    Therefore, $P(a, x, s) = P(a | x) P(s | x) P(x)$, which allows us to rewrite the above equation as
    \begin{align*}
        v^\pi (s) & = \sum_{a \in \mathcal{A}} \sum_{s' \in \mathcal{S}} P(a| x) \frac{P(s', a, x, s)}{P(a | x) P(x, s)}  G \\
                  & = \sum_{a \in \mathcal{A}} \sum_{s' \in \mathcal{S}} \frac{P(a | s', x, s) P(s', x, s)}{P(x, s)}  G \,. 
    \end{align*}
    By the law of total probability, we can now introduce the new variable $e$, which is the embedded action. Then
    \begin{align*}
        v^\pi (s) = & \sum_{a \in \mathcal{A}} \sum_{s' \in \mathcal{S}} \int_e \frac{P(a, e | s', x, s) P(s',x, s)}{P(x, s)}  G \,de                  \\
        =           & \sum_{a \in \mathcal{A}} \sum_{s' \in \mathcal{S}} \int_e P(e | x, s) \frac{P(a, e | s', x, s) P(s', x, s)}{P(e | x, s) P(x, s)} G \,de \,. 
    \end{align*}
    Since $e$ is uniquely determined by $x$, we can drop $s$ in $P(e|x,s)$, and thus
    \begin{align*}
        v^\pi (s) = & \sum_{a \in \mathcal{A}} \sum_{s' \in \mathcal{S}} \int_e P(e | x) \frac{P(a, e, s', x, s) }{P(e, x, s)}  G \,de \\
        =           & \sum_{a \in \mathcal{A}} \sum_{s' \in \mathcal{S}} \int_e P(e | x) P(a, s' | e, x, s) G \,de                     \\
        =           & \sum_{a \in \mathcal{A}} \sum_{s' \in \mathcal{S}} \int_e P(e | x) P(s' | a, e, x, s) P(a | x, e, s) \,de \,.  
    \end{align*}
    Using the previously derived Claims \ref{claim:drop_e_x_for_s_n} and \ref{claim:drop_e_x_for_a}, the above equation can be simplified to
    \begin{align*}
        v^\pi (s) & = \sum_{a \in \mathcal{A}} \sum_{s' \in \mathcal{S}} \int_e P(e | \phi(s)) P(s' | a, s) P(a | e) G \,de \,. 
    \end{align*}
    Since the function $f$, mapping $e$ to $a$, is deterministic by Assumption \ref{as:act_mapping} and only evaluates to 1 for a particular $a$ and 0 elsewhere, we can rewrite this further as
    \begin{align*}
        v^\pi (s) & = \sum_{a \in \mathcal{A}} \sum_{s' \in \mathcal{S}} \int_{f^{-1}(a)} P(e | \phi(s)) P(s' | a, s) G \,de \,. 
    \end{align*}
    Summarizing the terms, this becomes
    \begin{align*}
        v^\pi (s) & = \sum_{a \in \mathcal{A}} \int_{f^{-1}(a)} \pi_i(e | \phi(s)) Q^{\pi}(a, s) \,de \,. 
    \end{align*}
\vspace{-1em}
\end{proof}

\subsection{Proof of Lemma 2}
\label{appendix:prrof_L2}

\textbf{Lemma 2.}
\textit{Under Assumptions \ref{as:act_mapping}--\ref{as:unique_s_to_x}, for all deterministic functions $f$ and $\phi$ which map each point $e \in \mathcal{E}$ and $s \in \mathcal{S}$ to an action $a \in \mathcal{A}$ and to an embedded state $x \in \mathcal{X}$, and the internal policy $\pi_i$ parameterized by $\theta$, the gradient of the internal policy's performance function $\frac{\partial J_i(\theta)}{\partial \theta}$ equals the gradient of the overall policy's performance function $\frac{\partial J_o(\phi, \theta, f)}{\partial \theta}$, i.e.,}
\begin{align*}
    \frac{\partial J_i(\theta)}{\partial \theta} & = \frac{\partial J_o(\phi, \theta, f)}{\partial \theta}  \,.
\end{align*}


\begin{proof}
    Recall from Lemma \ref{lemma:v} that the overall policy is defined using the internal policy
    \begin{align*}
        \pi_o (a|s) & = \int_{f^{-1}(a)} \pi_i(e | \phi(s)) \,de \,.
    \end{align*}
    We can then define the performance function of the overall policy using the internal policy as
    \begin{align*}
        J_o(\phi, \theta, f) = \sum_{s \in \mathcal{S}} d_0(s) \sum_{a \in \mathcal{A}} & \int_{f^{-1}(a)} \pi_i(e | \phi(s)) Q^{\pi_o}(s, a) \,de \,.
    \end{align*}
    The gradient of this performance function w.r.t. $\theta$ is
    \begin{align}
        \label{eq:replacingQ}
        \begin{aligned}
            \frac{\partial J_o(\phi, \theta, f)}{\partial \theta} = \frac{\partial}{\partial \theta} & \Big[ \sum_{s \in \mathcal{S}} d_0(s) \sum_{a \in \mathcal{A}} \int_{f^{-1}(a)} \pi_i(e | \phi(s)) Q^{\pi_o}(s, a) \,de \Big ] \,.
        \end{aligned}
    \end{align}

    \begin{figure}[htbp]
        \centering
        \includegraphics[width=.6\linewidth]{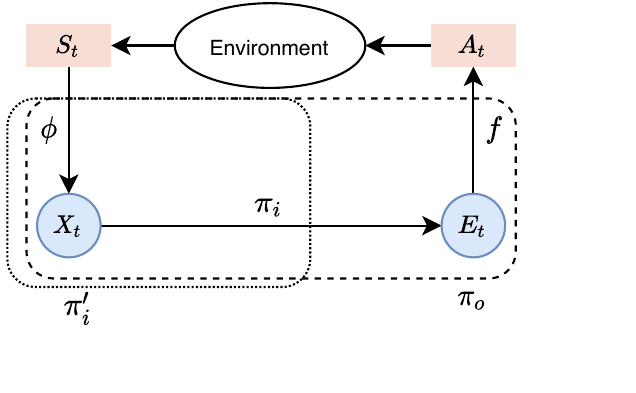}
        \vspace{-1.8em}
        \caption{Illustration of $\pi_i'$ in the combined model.}
        \label{fig:q_pi_i}
    \end{figure}


    Now consider another policy $\pi_i'$ as illustrated in Figure~\ref{fig:q_pi_i}, i.e., a policy that takes the raw state input $S_t$ and outputs $E_t$ in embedding space.
    By Assumption \ref{as:act_mapping}, $Q^{\pi_o}(s, a) = Q^{\pi_i'}(s, e)$ as $e$ is deterministically mapped to $a$.
    However, $Q^{\pi_i}(x, e) = \mathbb{E}_{s \sim \phi^{-1}(x)}[Q^{\pi_i'}(s, e)] = \mathbb{E}_{s \sim \phi^{-1}(x)}[Q^{\pi_o}(s, a)]$, where $\phi^{-1}(\cdot)$ is the inverse mapping from $x$ to $s$.
    Nevertheless, by Assumption \ref{as:unique_s_to_x}, $\mathbb{E}_{s \sim \phi^{-1}(x)}[Q^{\pi_o}(s, a)]=Q^{\pi_o}(s, a)$.
    Therefore, we can replace the overall $Q^{\pi_o}(a, s)$ in Equation~(\ref{eq:replacingQ}) by the internal $Q^{\pi_i}(\phi(s),e)$ as follows

    \begin{align*}
        \frac{\partial J_o(\phi, \theta, f)}{\partial \theta} = \frac{\partial}{\partial \theta} & \Big[ \sum_{s \in \mathcal{S}} d_0(s) \sum_{a \in \mathcal{A}}             \\
                                                                                                 & \int_{f^{-1}(a)} (\pi_i(e | \phi(s)) Q^{\pi_i}(\phi(s), e) \Big ] \,de \,.
    \end{align*}

    Then the deterministic mapping of $e$ to $a$ by function $f$ further allows us to replace the integral over $f^{-1}(a)$ and the summation over $\mathcal{A}$ with an integral over $e$.
    Moreover, by Assumption \ref{as:unique_s_to_x}, we have $d_0(s) = d_0(x=\phi(s))$.
    Therefore, the above can be rewritten as
    \begin{align*}
        \frac{\partial J_o(\phi, \theta, f)}{\partial \theta}
        = & \frac{\partial}{\partial \theta}\Big[ \int_{x\in\mathcal{X}} d_0(x) \int_{e\in \mathcal{E}} \pi_i(e | x) Q^{\pi_i}(x, e) \,de \,dx \Big ] \\
        = & \frac{\partial J_i(\theta)}{\partial \theta}\,.
    \end{align*}
\end{proof}

\subsection{Hyper Parameters}
\label{parameters:model_imp}
We conduct a gridsearch over different parameter combinations for each experiment and report the average per episode return and standard deviation achieved over $10$ runs with different random seeds, using the hyper parameters that perform the best in the gridsearch.
The learning rates for the actor and critic are searched over $\{0.001, 0.0003, 0.0001\}$.
We search the initialization value for the std. dev. of the actor over $\{0.3, 0.6, 0.8, 1.0\}$ and keep $\gamma=0.99$ and $\lambda=0.97$ fixed.
The batch size is searched over $\{256, 512\}$ and the network sizes for the actor and critic are searched over $\{(64, 64), (128, 128)\}$.
In discrete domains, we use a nearest neighbour mapping function to parameterize $f$, i.e., we map the output of the internal policy $\pi_i$ to the action that is closest in embedding space $\mathcal{E}$.
For continuous domains, $f$ is parameterized by a single-layer neural network with $64$ hidden units.
The dimensions of state and action embeddings respectively are fixed at $8$ and $2$ for the gridworld and searched over $\{8, 16\}$ and $\{2, 4\}$ for the slotmachine, $\{10, 20\}$ and $\{6, 10\}$ for the half-cheetah, $\{ 16, 32, 64\}$ and $\{6, 8, 10\}$ for the Ant-v2, and $\{8, 16, 32\}$ and $\{4, 8, 12\}$ for the recommender system.
We use the Adam optimizer for all experiments. 
We use tanh activation functions throughout and also use a tanh non-linearity to bound the range of the learned embeddings.
The total number of training steps can be seen in Figure \ref{fig:experiment_results}.
We adopt the same parameterizations for the RA and AE baseline to ensure comparability.
The parameters that performed best will be released separately along with the code upon publication of this paper.

\bibliographystyle{ACM-Reference-Format}
\bibliography{main}


\end{document}